\documentclass{article}
\usepackage{spconf,amsmath,graphicx}
\usepackage{url}
\graphicspath{{images/}}

\usepackage{amsthm}
\newtheoremstyle{boldnote}{\topsep}{\topsep}{\itshape}{}{\bfseries}{.}{ }{\thmname{#1}\thmnumber{ #2}\thmnote{ (#3)}}
\theoremstyle{boldnote}
\newtheorem{definition}{Definition}
\newtheorem{proposition}{Proposition}
\usepackage{booktabs}
\usepackage{siunitx}
\title{DuoPIC: Uncertainty-Aware Voxel-Level Digital Twins of Cerebral Perfusion with Transport PDEs in Acute Stroke}
\name{\begin{tabular}[t]{@{}c@{}}Anh Duong Tran$^{\star,*}$ \quad Quang-Anh N. D.$^{\dagger,*}$ \quad Toan Sy Pham$^{\star}$ \quad Trang H. Hoang$^{\dagger,\S}$ \\ Minh N. Do$^{\parallel,\sharp}$ 
\quad Duy Ton Mai$^{\ddagger,\P,\star}$ \quad Huan X. Nguyen$^{\dagger,\parallel}$\end{tabular}}

\address{
$^*$Equal contribution  \quad
$^\dagger$ London Digital Twin Research Centre, Middlesex University\\
$^\ddagger$ Bach Mai Stroke Centre, Bach Mai Hospital \quad $^{\P}$ Hanoi Medical University\\
$^\star$ Vietnam National University \quad
$^\sharp$ University of Illinois Urbana-Champaign \\
$^\parallel$ College of Engineering and Computer Science, VinUniversity \quad $^\S$ University of Alberta \quad
}

\let\oldthebibliography\thebibliography
\renewcommand{\thebibliography}[1]{\oldthebibliography{#1}\setlength{\itemsep}{1pt plus 0.5pt}\setlength{\parsep}{0pt}}
\begin{document}
\maketitle

\begin{abstract}
Digital twins of cerebral perfusion relate tracer dynamics to cerebral blood flow (CBF), blood volume (CBV) and mean transit time (MTT), but must be data-driven, physically consistent and uncertainty-aware. We introduce DuoPIC, a digital twin that needs no reference maps, in which a Windkessel flow model drives a tracer transport partial differential equation (PDE). A population-level neural closure divides each region into fast and slow pathways with tracer mass and MTT conserved by construction, and a tempered posterior quantifies the uncertainty of every estimate. Its voxel maps set the state of the art for lesion detection on ISLES'17 MR (AUC 0.745) and ISLES’24 CT (0.743). On held-out patients, the neural PDE reduced the compartment’s reconstruction error by 19\% and lifted 95\% MTT interval coverage from 88\% to 94\%.
Code: \url{https://github.com/detinatran/duopic}.
\end{abstract}


\section{Introduction}
\label{sec:intro}

\begin{figure*}[t]
\centering
\includegraphics[width=0.85\textwidth]{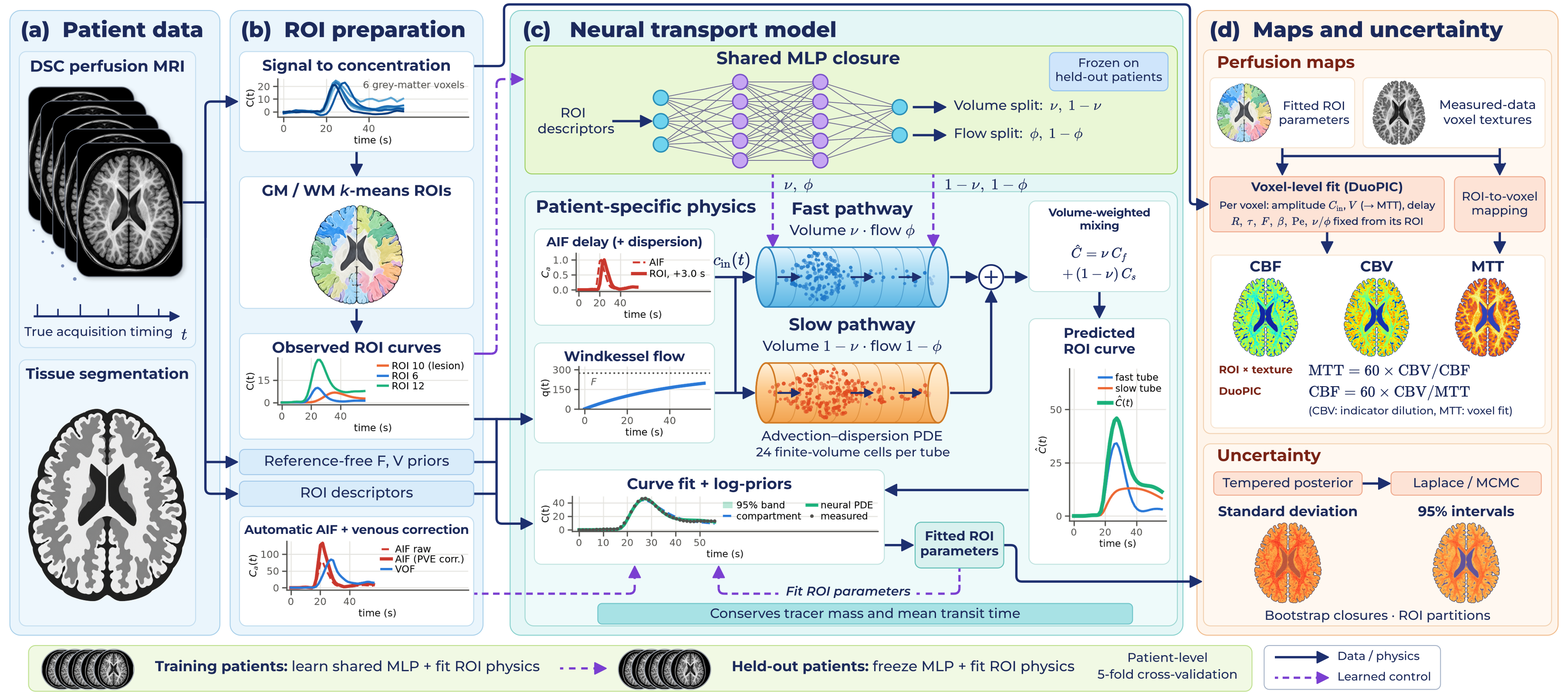}
\caption{Overview of DuoPIC. (a)~Input data. (b)~ROI preparation with reference-free priors and an automatic AIF. (c)~Neural transport model, in which Windkessel flow and a delayed, dispersed input drive a fast and a slow tube split by a shared network that conserves tracer mass and MTT. (d)~Voxel maps (ROI level $\times$ texture, or voxel fit) and their uncertainty.}
\label{fig:overview}
\end{figure*}

The World Health Organization reports 11.9 million new strokes worldwide in 2021 \cite{who2025stroke}. In acute ischemic stroke, perfusion imaging reveals hypoperfused but salvageable tissue, and in the DEFUSE 3 trial it selected patients benefiting from thrombectomy 6--16 hours after last known well \cite{albers2018defuse3}. Treatment decisions thus depend on reliable patient-specific estimates of cerebral blood flow (CBF), cerebral blood volume (CBV) and mean transit time (MTT). Clinical perfusion analysis deconvolves each tissue curve with an arterial input function (AIF) by truncated or block-circulant singular value decomposition \cite{ostergaard1996highres,wu2003osvd}, a model-free approach sensitive to noise and to the AIF. Model-based methods describe the microvasculature physiologically \cite{mouridsen2006bayesian}, treat perfusion as advection and diffusion \cite{liu2021piano}, or embed tracer kinetics in physics-informed neural networks \cite{devries2023spatiotemporal}, whereas deep networks predict the final infarct directly but output no physiological quantities \cite{nielsen2018prediction}. Yet mechanistic models are limited by their simplifications, and neural networks need labels and lack physical explanation. Physics-informed machine learning combines the two by letting a network learn only what the physical model leaves unexplained \cite{karniadakis2021physics,rackauckas2020universal}. A digital twin is a physics-based model personalized using patient-specific data and incorporating physiologically meaningful parameters to simulate unobserved conditions, such as a restored inflow \cite{katsoulakis2024digital,briggs2026cerebrlsim}.  
 
HemoPIC \cite{lee2026hemopic} brought digital twins to brain perfusion by personalizing flow and tracer models for every scan region. However, this method has several limitations, including inconsistent units, reliance on priors derived from reference deconvolution maps, an inability to capture capillary transit-time heterogeneity due to its well-mixed regional compartment assumption \cite{ostergaard2013capillaries}, and a lack of uncertainty quantification. We address these limitations by using a shared neural network across patients to learn only what the physics cannot identify (Fig.~\ref{fig:overview}). Our main contributions are threefold.
(i)~We propose DuoPIC, a reference-free and uncertainty-aware perfusion digital twin framework that couples patient-specific haemodynamics with transport physics from region to voxel level.
(ii)~We introduce a physics-preserving neural closure, shared across patients, that learns transit-time heterogeneity while provably conserving tracer mass, CBF, CBV and MTT.
(iii) DuoPIC voxel maps achieve state-of-the-art lesion detection on ISLES'17 MR and 145 ISLES'24 CT perfusion patients.

\section{Proposed DuoPIC}
\label{sec:method}
\setlength{\abovedisplayskip}{3pt}\setlength{\belowdisplayskip}{3pt}%
From the dynamic susceptibility contrast (DSC) concentration $C(\mathbf{x},t_n)=-\mathrm{TE}^{-1}\ln(S/S_0)$ (echo time TE) at frames $t_n=n\Delta$, DuoPIC estimates voxel CBF, CBV and MTT with posterior SD and 95\% intervals, without reference maps or lesion labels, as in Definition~\ref{def:id}. Voxel curves are too noisy for such a model, so DuoPIC builds on HemoPIC, where k-means clusters grey- and white-matter voxels by curve features and position into homogeneous regions of interest (ROIs).
\par\smallskip\noindent\begin{minipage}{\columnwidth}
\begin{definition}[Indicator dilution]
\label{def:id}
With $\kappa_0=100k_H/\rho$ (haematocrit factor $k_H$, tissue density $\rho$), tissue responds to the arterial curve $C_a$ as $C=\frac{\mathrm{CBF}}{60\kappa_0}(C_a*\mathcal{R})$ (CBF in ml/100\,g/min, $t$ in s), with residue function $\mathcal R$ and transit-time density $h=-\mathcal R'$, so that
\begin{equation}
\mathrm{CBV}=\kappa_0\frac{\int C\,dt}{\int C_a\,dt}, \mathrm{MTT}=\int_0^\infty t\,h\,dt=60\,\frac{\mathrm{CBV}}{\mathrm{CBF}}.
\label{eq:cvt}
\end{equation}
\end{definition}
\end{minipage}

\begin{figure}[t]
\centering
\includegraphics[width=\columnwidth]{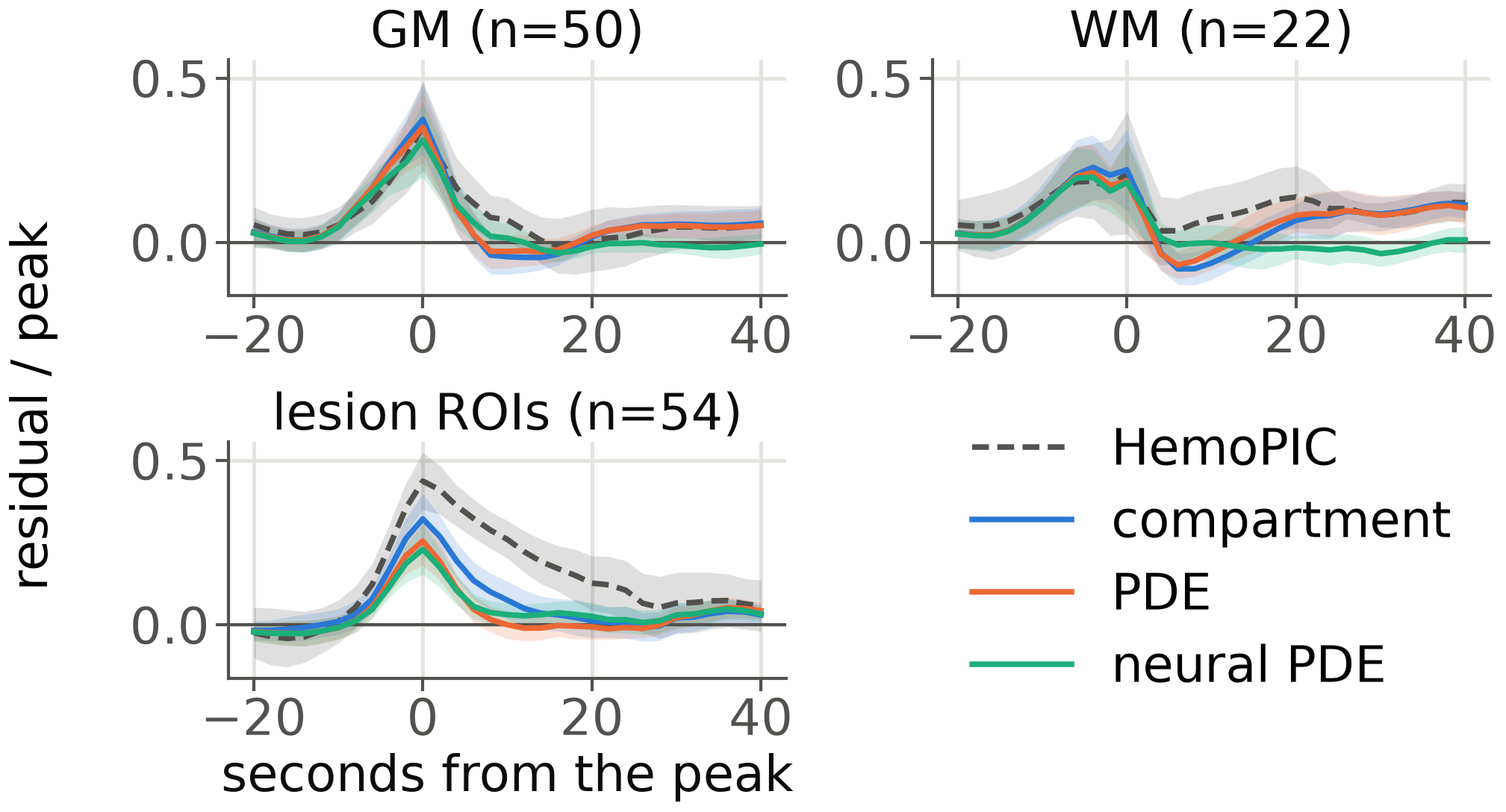}
\caption{Mean residual (measured $-$ fitted, relative to the peak) around the bolus peak, $\pm2$ standard errors, for HemoPIC \cite{lee2026hemopic} and DuoPIC; GM, WM: lesion-free ROIs; lesion: ROIs over 20\% infarcted; $n$: number of ROIs.}
\label{fig:residuals}
\end{figure}

\begin{definition}[HemoPIC regional model \cite{lee2026hemopic}]
\label{def:hemopic}
For every ROI, the outflow follows a two-element Windkessel, $q(t)=F+(p_0/R-F)\,e^{-t/\tau}$, and tracer kinetics follow a well-mixed compartment model,
\begin{equation}
\frac{dC}{dt}=k(t)\Big(\frac{F}{q(t)}\,c_{\mathrm{in}}(t)-C\Big),\qquad k(t)\propto\frac{q(t)}{V},
\label{eq:tank}
\end{equation}
driven by the arterial input $c_{\mathrm{in}}=C_{\mathrm{in}}\hat a$, with $\hat a$ the normalized input curve and exponential transit times, $h(t)=ke^{-kt}$.
\end{definition}
\noindent Its parameters $\theta$ minimize
\begin{equation}
L(\theta)=\frac{N(\theta)}{D}+\sum_p\lambda_p\big(\ln p-\ln p^0\big)^2,
\label{eq:loss}
\end{equation}
with $N=\sum_n w_n(\hat C_n(\theta)-c_n)^2$ and $D=\sum_n w_nc_n^2$ for the ROI curve $c$, so ROIs are comparable, $w_n$ emphasizes the bolus and log-space priors keep parameters positive. Voxel maps are the ROI level of CBF or CBV times a voxel texture.
 
\subsection{Reference-free regional model}
\label{ssec:corrections}
The model runs on the true time axis of each scan, with $\Delta$ the repetition time from the DICOM metadata (1.4-1.8\,s), as a nominal 1\,s would bias all time constants by 40-80\%. The rate $k=q/(60V)$ is defined per second for a flow in ml/100\,g/min, so that $\mathrm{MTT}=1/k|_{q=F}=60V/F$ matches Eq.~\eqref{eq:cvt} and the curve shape constrains $F$ and $V$ rather than their priors. Eq.~\eqref{eq:tank} is integrated exactly per frame, $C_{n+1}=a_nC_n+(1-a_n)(F/q_n)c_{\mathrm{in},n}$ with $a_n=e^{-k_n\Delta}$.
 
An absolute MTT scale requires an arterial input, since tissue curves have already crossed the capillary bed. The AIF $C_a$ is the earliest shape cluster of clean, high-area bolus curves, corrected for partial volume by the venous output function, $C_a\leftarrow\mathrm{clip}(\int C_v/\int C_a,1,5)\,C_a$. As ischemic tissue receives the bolus late and dispersed, each ROI has its own delay and dispersion \cite{calamante2000delay},
\begin{equation}
c_{\mathrm{in}}(t)=C_{\mathrm{in}}\,(h_\beta*\hat a)(t-\delta),\quad h_\beta(t)=\beta^{-1}e^{-t/\beta},
\label{eq:input}
\end{equation}
applied exactly as a first-order recursion. To avoid circularity, the priors are the indicator-dilution volume and maximum-slope flow of the baseline-subtracted ROI and arterial curves $d$, $d_a$, in absolute units,
\begin{equation}
V^0=\kappa_0\,\frac{\sum_n d_n}{\sum_n d_{a,n}},\qquad F^0=60\,\kappa_0\,\frac{\max_n \dot d_n}{\max_n d_{a,n}}.
\label{eq:priors}
\end{equation}
CBV is computed by voxel indicator dilution (ID-CBV) with an oSVD-based CBF texture. Poor fits ($N/D>0.15$) get four random restarts, and all fits are polished by L-BFGS-B.
 
\subsection{Transport PDE and neural closure}
\label{ssec:pde}
The compartment still misses the bolus peak in all tissues (Fig.~\ref{fig:residuals}), as with exponential transit times tracer leaves on arrival. We therefore model each ROI as a transport tube.
\begin{definition}[Transport tissue model]
\label{def:pde}
The tracer of an ROI moves through a tube, $s\in[0,1]$, with plug flow and longitudinal dispersion at P\'eclet number Pe,
\begin{equation}
\frac{\partial c}{\partial t}=k(t)\Big(-\frac{\partial c}{\partial s}+\mathrm{Pe}^{-1}\,\frac{\partial^2 c}{\partial s^2}\Big),\qquad C(t)=\int_0^1 c\,ds,
\label{eq:pde}
\end{equation}
with the Danckwerts inlet condition $c-\mathrm{Pe}^{-1}\partial_s c=u:=(F/q)\,c_{\mathrm{in}}$ \cite{danckwerts1953continuous} and a zero-gradient outlet.
\end{definition}
\noindent With $N=24$ finite volumes $v_i$ and upwind convection, $\dot{\mathbf c}=k(\mathbf A\mathbf c+\mathbf b u)$, $\mathbf b=\mathbf e_1/v_1$ and $C=\mathbf v^{\top}\mathbf c$. The tridiagonal $\mathbf A$ is similar to a symmetric matrix, so $C$ is a sum of $N$ exponential modes advanced exactly per frame from one differentiable eigendecomposition per ROI.

\begin{figure}[t]
\centering
\includegraphics[width=0.8\columnwidth]{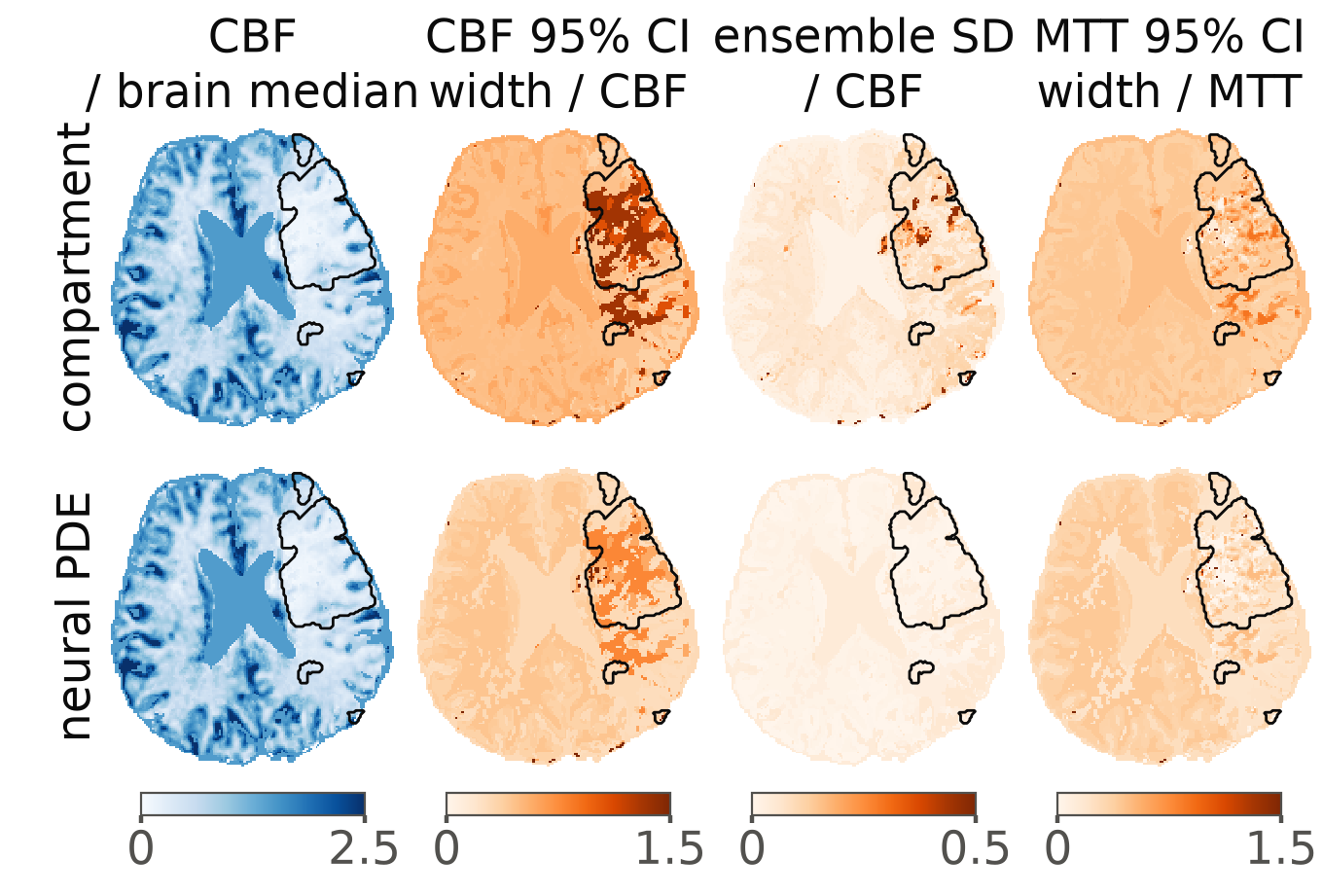}
\caption{Uncertainty maps for the compartment (top) and the neural PDE (bottom), with the final infarct in black. CBF is divided by its brain median and the others by the estimate.}
\label{fig:maps}
\end{figure}

A single tube only narrows the transit-time distribution (coefficient of variation $\mathrm{CV}\ge1/\sqrt N$), whereas ROIs mixing fast and slow tissue need a wider one (capillary transit-time heterogeneity \cite{ostergaard2013capillaries}). The neural closure \cite{rackauckas2020universal} splits each ROI into two parallel tubes sharing $u$ and Pe, with volume and flow fractions $\nu_m$, $\phi_m$ (each summing to 1), rates $k\phi_m/\nu_m$ and $C=\sum_m\nu_mC_m$. The following result ensures that it cannot absorb the physics.
\begin{proposition}[Conservation and map invariance]
\label{prop:invariance}
For any Pe and split, (i) tracer mass is conserved and the mean transit time is $1/k$, so CBF, CBV and $\mathrm{MTT}=60V/F$ keep the meaning of Eq.~\eqref{eq:cvt}, (ii) with $\nu=\phi$, a single cell or $\mathrm{Pe}\to0$ recovers Eq.~\eqref{eq:tank}, and (iii) transit-time spread can only grow,
\begin{equation}
\mathrm{CV}^2_{\mathrm{mix}}=(1+\mathrm{CV}^2_{\mathrm{tube}})\textstyle\sum_m\nu_m^2/\phi_m-1\ \ge\ \mathrm{CV}^2_{\mathrm{tube}},
\label{eq:cv}
\end{equation}
with equality iff $\nu=\phi$.
\end{proposition}
\noindent\emph{Proof sketch.} $\mathbf v^{\top}\mathbf A=-\mathbf e_N^{\top}$ and $\mathbf A\mathbf 1=-\mathbf b$ give conservation and a unit-mean transit time per tube. A particle takes tube $m$ with probability $\phi_m$ and stays $\nu_m/\phi_m$ times longer, and Cauchy-Schwarz gives $\sum_m\nu_m^2/\phi_m\ge1$. \hfill\qedsymbol
 
\noindent As one curve cannot identify the split next to eight physical parameters, an MLP $f_\psi$ ($9{\to}32{\to}32{\to}2$) shared across patients maps label-free descriptors $\mathbf z_r$ (tissue class, $\ln F^0$, $\ln V^0$, time to peak and width relative to the input) to $(\nu_1,\phi_1)=0.05+0.9\,\sigma(f_\psi(\mathbf z_r))$. In patient-level 5-fold cross-validation, it is trained jointly with the physics of the training patients by minimizing $\sum_r[L_r(\theta_r;f_\psi(\mathbf z_r))+\eta(f_{\psi,1}-f_{\psi,2})^2]$, with $\eta=10^{-3}$ favouring one rate, and frozen on held-out patients.

\begin{figure*}[t]
\centering
\includegraphics[width=0.9\textwidth]{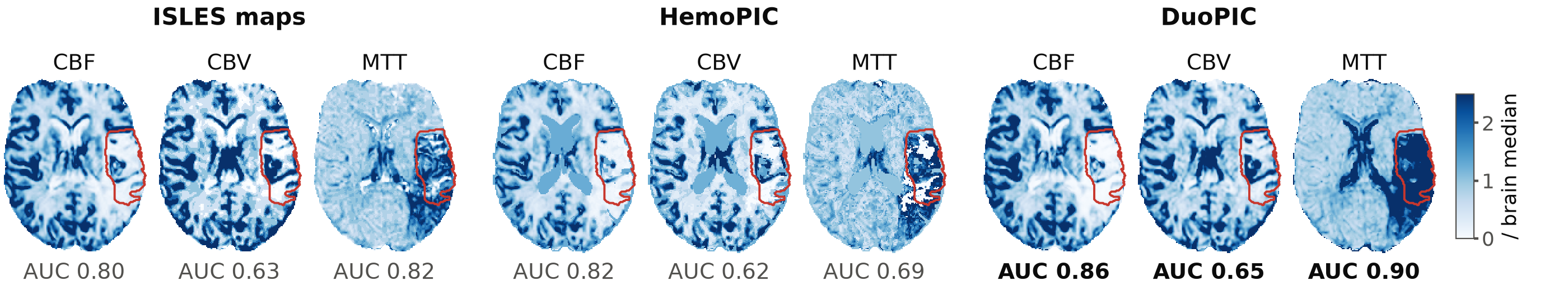}
\caption{CBF, CBV and MTT of one ISLES'17 patient relative to the brain median, with the final infarct in red, from the ISLES maps, HemoPIC \cite{lee2026hemopic} and DuoPIC, with the lesion AUC of each map (best in bold).}
\label{fig:hemopic}
\end{figure*}

\subsection{Voxel-level fit}
\label{ssec:voxel}
The final maps fit every voxel curve with its ROI model, which fixes what a noisy curve cannot identify ($R$, $\tau$, $F$, $\beta$, Pe, split). Each voxel fits its amplitude, its shape through $V$ (only $F/V$ enters it) and its delay, $\vartheta_{\mathbf x}=(\ln V_{\mathbf x},\ln C_{\mathrm{in},\mathbf x},\delta_{\mathbf x})$, by minimizing
\begin{equation}
L_{\mathbf{x}}=\frac{N_{\mathbf{x}}}{D_{\mathbf{x}}}+(\vartheta_{\mathbf{x}}-\vartheta_r)^{\top}\Lambda\,(\vartheta_{\mathbf{x}}-\vartheta_r),
\label{eq:voxel}
\end{equation}
with diagonal priors $\Lambda$ around the ROI values $\vartheta_r$, fixed in advance. Reusing the ROI modes, batched Levenberg-Marquardt fits the ${\sim}2\cdot10^{5}$ voxels of a patient in about 100\,s on one CPU thread. Since shape identifies MTT better than amplitude identifies volume, $\mathrm{MTT}_{\mathbf{x}}=60V_{\mathbf{x}}/\bar q$, with $\bar q$ the late ROI outflow, is combined with ID-CBV as $\mathrm{CBF}_{\mathbf{x}}=60\,\mathrm{CBV}_{\mathbf{x}}/\mathrm{MTT}_{\mathbf{x}}$, obeying Eq.~\eqref{eq:cvt}.
 
\subsection{Uncertainty quantification}
\label{ssec:uq}
With a suitable temperature, the loss becomes a likelihood,
\begin{equation}
p(\theta\mid c)\propto e^{-\kappa L(\theta)},\quad \kappa=\frac{n_{\mathrm{eff}}\,D}{2\hat N},\quad n_{\mathrm{eff}}=\frac{T}{2\sqrt{\pi}\,\sigma_s},
\label{eq:posterior}
\end{equation}
so $\kappa N/D$ is a Gaussian negative log-likelihood (up to a constant) with the fitted residual $\hat N$, $n_{\mathrm{eff}}$ counts the frames left independent by smoothing ($\sigma_s=1.8$ frames), and the mode is the fit itself. ROI and PDE parameters are sampled by random-walk Metropolis preconditioned by a Laplace approximation, and voxels add their own Laplace approximation via the implicit function theorem. Bootstrap networks and a pipeline ensemble (k-means seeds and sizes) measure network and partition variability. Checks use hidden frames, uncertain voxels and semi-synthetic refits.
 

\section{Experimental setup}
\label{sec:experimental_setup}
\textbf{Data.} We used the 43 ISLES'17 subjects \cite{winzeck2018isles} with DSC perfusion and final infarct masks, preprocessed as in HemoPIC \cite{lee2026hemopic} with four SynthSeg tissue classes \cite{billot2023synthseg}. As an external test, we utilised the 149 ISLES'24 subjects \cite{riedel2024isles} with CT perfusion and follow-up MRI masks (145 with infarct in view), registered in non-contrast CT, baseline-subtracted and smoothed in plane (3-voxel Gaussian) for all methods. ISLES'17 lesion detection used the 36 subjects whose AIF were earlier, narrower and higher than the tissue mean.

\noindent\textbf{Metrics.} We reported the normalized root-mean-square error (nRMSE) of the tracer reconstruction, the lesion ROC AUC and average precision (AP) of CBF (low), CBV (low) and MTT (high) relative to the brain median with paired Wilcoxon tests and patient-level bootstrap intervals, and the Akaike information criterion (AIC) of the tissue models.

\noindent\textbf{Implementation details.} The closure was trained with Adam and a cosine schedule for 1500 steps (learning rates $2\cdot10^{-2}$ physics, $2\cdot10^{-3}$ network), held-out patients were refitted for 1000 steps, and the posterior run four chains of 6000 steps after 3000 burn-in, keeping every fourth (1500 draws per chain).

\section{Experimental Results}
\label{sec:results}
 
\noindent\textbf{Model fit and physiology.} HemoPIC's CBF and CBV largely returned its deconvolution priors: in about 90\% of its ROIs, the posterior SD of $F$ and $V$  equaled the prior SD. With consistent units, the data constrained both in over 90\% of the ROIs. The compartment still missed 22-37\% of the signal at the bolus peak (HemoPIC 21-44\%, Fig.~\ref{fig:residuals}). The PDE lowered the brain nRMSE, and the neural PDE lowered it further on held-out patients in all five folds (Table~\ref{tab:pde}); on hidden frames, it predicted the curve better than the compartment in 65\% of the ROIs (a single tube in 36\%). In normal grey matter, its median CBF (52\,ml/100\,g/min), CBV (4.2\,ml/100\,g) and MTT (4.9\,s) were physiological (HemoPIC 22, 1.4 and 4.6).

\begin{table}[!t]
\centering
\setlength{\abovecaptionskip}{4pt}
\caption{Held-out diagnostics. Parameter coverage: semi-synthetic; width: median 95\% interval / value; nMSE: hidden-frame error, ROI median; -: not evaluated.}
\label{tab:pde}
\small
\setlength{\tabcolsep}{0pt}
\setlength{\aboverulesep}{1pt}
\setlength{\belowrulesep}{1.5pt}
\begin{tabular*}{\columnwidth}{@{\extracolsep{\fill}}p{0.42\columnwidth} *{3}{>{\centering\arraybackslash}p{0.19333\columnwidth}}@{}}
\hline
\textbf{Metric} & \textbf{Compart.} & \textbf{PDE} & \textbf{Neural PDE} \\
\hline
Brain nRMSE $\downarrow$ & 0.116 & 0.108 & \bfseries 0.094 \\
Lesion nRMSE $\downarrow$ & 0.271 & 0.277 & \bfseries 0.265 \\
\hline
CBF coverage (\%) $\to$95 & \bfseries 97.0 & {-} & 98.2 \\
MTT coverage (\%) $\to$95 & 88.1 & {-} & \bfseries 93.8 \\
$\tau$ coverage (\%) $\to$95 & 70.3 & {-} & \bfseries 78.4 \\
CBF width (\%) $\downarrow$ & 85 & {-} & \bfseries 72 \\
MTT width (\%) $\downarrow$ & 94 & {-} & \bfseries 71 \\
MTT median $|$rel.\ err.$|$ (\%) $\downarrow$ & 15.7 & {-} & \bfseries 8.9 \\
\hline
Hidden 95\% cov.\ (\%) $\to$95 & 84.0 & {-} & \bfseries 84.8 \\
Hidden nMSE $\downarrow$ & 0.338 & 0.333 & \bfseries 0.197 \\
\addlinespace[2pt]
AIC below comp.\ (\%) $\uparrow$ & {-} & 27 & \bfseries 64 \\
\hline
\end{tabular*}
\end{table}

 \begin{table}[!t]
\centering
\setlength{\abovecaptionskip}{4pt}
\caption{Lesion AUC, 36 ISLES'17 patients with an accepted AIF, Mean: map-averaged AUC; AP: average precision.}
\label{tab:auc}
\small
\setlength{\tabcolsep}{2pt}
\setlength{\aboverulesep}{1.5pt}
\setlength{\belowrulesep}{2pt}
\begin{tabular*}{\columnwidth}{@{\extracolsep{\fill}}l *{5}{S[table-format=1.3]}@{}}
\hline
\textbf{Method} & {\textbf{CBF}} & {\textbf{CBV}} & {\textbf{MTT}} & {\textbf{Mean}} & {\textbf{AP}} \\
\hline
ISLES maps & 0.745 & 0.584 & 0.700 & 0.676 & 0.085 \\
HemoPIC \cite{lee2026hemopic} & 0.773 & 0.592 & 0.722 & 0.696 & 0.104 \\
oSVD \cite{wu2003osvd} & 0.782 & \bfseries 0.623 & 0.759 & 0.721 & 0.109 \\
\hline
DuoPIC, compartment & 0.790 & 0.552 & 0.766 & 0.703 & 0.113 \\
DuoPIC, PDE & 0.787 & 0.555 & 0.774 & 0.705 & 0.113 \\
DuoPIC, neural PDE & 0.788 & 0.561 & 0.768 & 0.706 & 0.113 \\
DuoPIC, ID-CBV & 0.807 & \bfseries 0.623 & 0.745 & 0.725 & 0.114 \\
\textbf{DuoPIC} & \bfseries 0.810 & \bfseries 0.623 & \bfseries 0.801 & \bfseries 0.745 & \bfseries 0.120 \\
\hline
\end{tabular*}
\end{table}

\noindent\textbf{Uncertainty.} In the semi-synthetic check (Table~\ref{tab:pde}), only the MTT and relaxation-time ($\tau$) intervals were too narrow, and the neural PDE brought both closer to 95\%. It also narrowed the median 95\% CBF interval from 116\% of the value (HemoPIC) to 72\%, and its network uncertainty (0.6\% of CBF) was small next to the posterior one (16.9\%). Removing the most uncertain voxels (Fig.~\ref{fig:maps}) raised the lesion AUC of MTT from 0.755 to 0.787.
  
\noindent\textbf{Lesion detection.} With the same AIF, oSVD already outperformed HemoPIC and the ISLES maps in mean lesion AUC (Table~\ref{tab:auc}), whereas the better-fitting PDE and neural PDE barely changed it. Without any tuning, DuoPIC achieved the best mean AUC, $0.023$ above oSVD (95\% CI $[0.011,0.035]$), and the best MTT map ($+0.042$, $p=0.002$). It improved every map over HemoPIC ($p<0.001$) and the mean AUC in 31 of 36 patients (Fig.~\ref{fig:hemopic}). On 145 ISLES'24 CT patients, it reached a mean AUC of 0.743 (best oSVD setting 0.691, HemoPIC 0.688, vendor maps 0.687, $p<10^{-20}$ each) and an AP of 0.120 (HemoPIC 0.104, vendor maps 0.088).

\section{Conclusion}
\label{sec:conclusion}
DuoPIC makes a perfusion digital twin reference-free and uncertainty-aware. Its neural transport PDE better predicts unseen frames with physical CBF, CBV and MTT, and its tuning-free voxel fit achieves state-of-the-art lesion AUC and AP. Absolute values still depend on the automatic AIF, and the reference infarct also reflects treatment.

\newpage

\section{Compliance with ethical standards}
\label{sec:ethics}
This study used open-access ISLES'17 \cite{winzeck2018isles} and ISLES'24 \cite{riedel2024isles} data, whose licenses require no ethical approval.

\section{Acknowledgments}
\label{sec:acknowledgments}
This work was supported in part by the UK British Council ISPF programmes: the Early Career Fellowships (grant ID: 1860), the International Research Collaborations with Vietnam (grant ID: 1203701580), and the Going Global Partnerships 2025 under the Transnational Education (TNE) Exploratory Grants scheme (Grant ID: TNE2025-4829).

\bibliographystyle{IEEEbib}
\bibliography{bib/introduction_refs}
\end{document}